\documentclass[letterpaper, 10 pt, conference]{ieeeconf}  

\IEEEoverridecommandlockouts                              

\usepackage{amsmath,amssymb}
\usepackage{amstext}

\usepackage{amsthm}
\usepackage{amsfonts}
\usepackage{mathtools}
\usepackage{thmtools}
\usepackage{thm-restate}
\usepackage{latexsym}
\usepackage{bm}
\usepackage{bbm}
\usepackage{dsfont}
\usepackage{nicefrac}
\usepackage[super]{nth}
\usepackage{booktabs}
\usepackage{multirow}
\usepackage{siunitx}
\usepackage{graphicx}
\usepackage{subcaption}
\usepackage{xcolor}
\usepackage[normalem]{ulem}
\usepackage[table]{xcolor}

\usepackage{algorithm}
\usepackage{algpseudocode}

\theoremstyle{plain}

\newtheorem{theorem}{Theorem}

\theoremstyle{definition}

\theoremstyle{remark}

\usepackage{algorithm}
\usepackage{algpseudocode}
\usepackage{amsmath}

\title{\LARGE \bf
Constrained Differentiable Policy Learning for Safe\\
Quadrotor Navigation
}

\author{
Yan Rui Tan$^{1*}$,
Praveen Krishnapur$^{1*}$,
Zeyu Feng$^{2*}$,
Han Bao$^{1}$,
Yew-Soon Ong$^{2}$,
Fan Shi$^{1}$
\thanks{$^{*}$These authors contributed equally to this work.}%
\thanks{$^{1}$Yan Rui Tan, Praveen Krishnapur, Han Bao,
and Fan Shi are with the National University of Singapore,
Singapore.
{\tt\small
\{yr.tan, bao\_han\}@u.nus.edu;
\{praveen\_krish, fan.shi\}@nus.edu.sg}}%
\thanks{$^{2}$Zeyu Feng and Yew-Soon Ong are with
the Agency for Science, Technology and Research (A*STAR),
Singapore.
{\tt\small
\{feng\_zeyu, Ong\_Yew\_Soon\}@a-star.edu.sg}}%
}

\begin{document}

\maketitle
\thispagestyle{empty}
\pagestyle{empty}

\begin{abstract}
Differentiable simulation enables efficient policy optimization through analytic gradients, yet safety is commonly incorporated as a weighted penalty, requiring manual tuning of the trade-off between task performance and constraint satisfaction. We instead formulate obstacle avoidance as an explicit constraint and train quadrotor policies using primal--dual gradient descent--ascent through a differentiable simulator. We introduce two key components to enable effective constrained training. First, we propose a bounded clearance-based constraint function that provides informative gradients around the safety boundary while preventing large violations from destabilizing primal--dual optimization. Second, we introduce log-mean-exp (LME) aggregation to smoothly emphasize critical per-step safety violations across a trajectory and training batch within a single constraint governed by one dual variable. Together, these components enable the desired safety level to be specified directly while balancing constraint satisfaction and task completion. We theoretically prove that the proposed optimization with differentiable simulation converges to feasible policy under mild assumptions. Experiments demonstrate improved safety--performance trade-offs over fixed-penalty and existing policy-learning baselines, as well as generalization to denser out-of-distribution environments. Finally, the learned policy transfers zero-shot to Gazebo/PX4 and a physical quadrotor without additional fine-tuning.
\end{abstract}

\section{INTRODUCTION}

Differentiable simulators expose analytic gradients of the dynamics, allowing policy objectives to be optimized directly by backpropagation through time (BPTT) rather than by high-variance score-function estimates \cite{xu2022accelerated, luo2025residual, xing2025stabilizing}. For quadrotors, whose nominal flight dynamics are smooth and contact-free, differentiable simulation provides an attractive alternative to model-free reinforcement learning, enabling agile flight and vision-based navigation with substantially fewer environment interactions \cite{zhang2025learning, su2026vector}.
This paradigm is particularly well suited to real world tasks including obstacle avoidance, allowing environmental vehicle--obstacle clearance to be differentiated directly with respect to the state.

\begin{figure}[t]
    \centering
    \includegraphics[width=0.9\columnwidth]{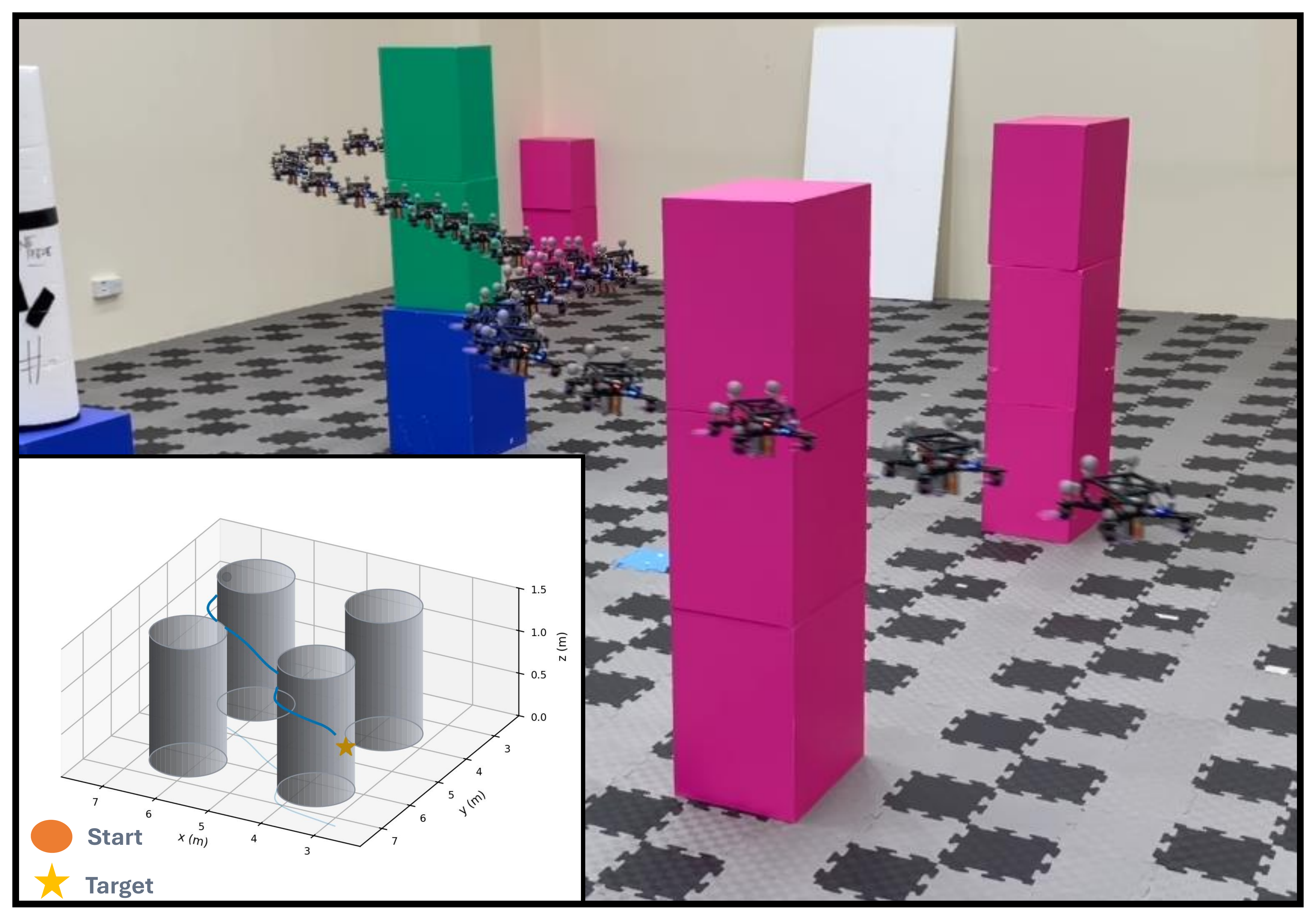}
    \caption{Zero-shot deployment of our primal--dual policy on a physical quadrotor. Due to practical limitations in reproducing the exact obstacle geometry, the physical obstacles serve as proxies for the virtual cylindrical obstacles perceived by the policy. The inset shows the corresponding trajectory in the policy's perceived environment.}
    \label{fig:your_figure}
\end{figure}

Yet safety is often incorporated into these objectives through a fixed weighted penalty \cite{zhang2025learning, deng2026autonomous}, coupling obstacle avoidance and task performance through a manually selected weight. A weak penalty may permit unsafe trajectories, whereas a strong penalty can compromise task performance; moreover, the appropriate weighting may vary across environments and tasks.
A constrained formulation offers a natural alternative by separating safety requirements from the task objective and explicitly specifying a desired constraint level. Unlike a fixed penalty, for which introducing such a threshold only shifts the objective by a constant, constrained optimization algorithms like primal--dual is capable of adaptively adjusting the penalty, allowing the specified constraint level to directly guide policy optimization.

Constrained policy optimization is well established in model-free safe reinforcement learning. While conventional methods for constrained Markov Decision Processes (CMDP) typically constrain
expected cumulative costs \cite{achiam2017constrained,stooke2020responsive},
instantaneous-constraint methods can provide safety guarantees at each decision
step under their respective assumptions \cite{shi2023near}. Such model-free
approaches, however, do not exploit analytic gradients through the dynamics and
are generally sample inefficient. More recently, CGPO extends constrained
optimization to differentiable simulation using analytic trajectory gradients
and adaptive trust-region updates to bound constraint violation
\cite{dai2024safe}. However, this does not guarantee strict feasibility, and
its finite-horizon cumulative constraint does not directly ensure safety at
each timestep. Since our training occurs entirely in simulation, intermediate
violations carry no physical consequence; instead, the key objective is to
guide the final policy toward a safety requirement that reflects critical
per-step obstacle avoidance while remaining compatible with analytic policy
gradients.

To this end, we adopt a primal--dual formulation and introduce two components
tailored to differentiable obstacle avoidance. First, we use log-mean-exp
(LME), a smooth approximation of the maximum per-step constraint cost, to
aggregate critical per-step safety into a single trajectory-level constraint
governed by one dual variable. Its differentiable form allows analytic safety
gradients to be backpropagated through the simulated dynamics. Second, we
introduce a bounded clearance-based constraint cost that provides informative
gradients near the safety boundary while limiting the influence of severe
obstacle penetration. Together, these components provide a smooth,
well-behaved constraint for differentiable primal--dual policy optimization.
Experimental results demonstrate the effectiveness of our approach and significant performance improvement in challenging out-of-distribution tests.

    
    

Our main contributions are:
\begin{itemize}
    \item A differentiable safety-constraint formulation combining a bounded clearance-based cost with LME aggregation for stable primal--dual policy optimization.
    
    \item A theoretical analysis on the convergence to feasible policy of our approach under mild assumptions.
    
    \item Simulation and real-world validation demonstrating favorable
    safety--performance trade-offs and zero-shot
    sim-to-real transfer.
\end{itemize}

\section{RELATED WORKS}

\subsection{Differentiable Policy Learning and Quadrotor Navigation}

Learning-based quadrotor navigation has been approached through imitation of
optimization-based or privileged experts
\cite{stevvsic2018sample, loquercio2021learning}, as well as direct policy
optimization \cite{yu2025depth}. Differentiable simulation provides an
increasingly attractive alternative by backpropagating first-order gradients
through simulated dynamics, enabling substantially more sample-efficient policy
optimization. SHAC addresses gradient degradation over long horizons through
truncated backpropagation and value-function bootstrapping
\cite{xu2022accelerated}, while subsequent work has considered partially
differentiable rewards \cite{li2025abpt}, optimization stability
\cite{xing2025stabilizing}, surrogate gradients \cite{heeg2025learning},
and online policy learning \cite{pan2026learning,ren2026learning}.

For quadrotor obstacle avoidance specifically, differentiable simulation has
enabled vision-based navigation policies using simplified vehicle dynamics
\cite{zhang2025learning,hu2025seeing}. While these methods demonstrate
efficient learning and successful navigation, obstacle avoidance is typically
incorporated through weighted objective terms, leaving the resulting
safety--performance trade-off dependent on the chosen weights. Our work
instead treats obstacle avoidance as an explicit constraint within
differentiable policy optimization.





\subsection{Constrained and Safe Policy Optimization}

Constrained policy optimization broadly follows two paradigms. Lagrangian
primal--dual methods, such as PPO/TRPO-Lagrangian, incorporate constraints
through adaptive dual variables, but can exhibit constraint overshoot and
oscillatory training, motivating improved dual updates such as PID-Lagrangian
\cite{zhang2020first,stooke2020responsive}. Alternatively, methods such as CPO
and FOCOPS explicitly constrain the policy update through trust-region or
policy-space optimization \cite{achiam2017constrained,zhang2020first}.
Regardless of the optimizer, conventional CMDP formulations commonly constrain
expected cumulative safety costs, which can permit localized violations along
a trajectory \cite{liu2020ipo,shi2023near}. Instantaneous-constraint methods have been proposed to enforce hard safety requirements at each decision step \cite{shi2023near}. However, these methods are not formulated for analytic policy optimization through differentiable dynamics.

More recently, CGPO combines analytic trajectory gradients with constrained
trust-region updates in differentiable simulation and establishes BPTT- and
SHAC-Lagrangian baselines \cite{dai2024safe}. CGPO constrains cumulative
finite-horizon safety cost, whereas our work focuses on the constraint
representation for obstacle avoidance, where localized violations may directly
correspond to collisions. We therefore use LME to construct a smooth,
worst-case-oriented trajectory constraint that emphasizes critical per-step
safety, together with a bounded clearance cost for stable primal--dual
optimization.

\subsection{Smooth Constraint Representations for Collision Avoidance}

Smooth approximations of minimum and maximum operators, particularly
log-sum-exp (LSE), have been used to combine multiple safety constraints
into differentiable representations. Molnar and Ames use smooth composition
to construct a single control barrier function for complex safety
specifications \cite{molnar2023composing}, while related formulations use
LSE to construct smooth collision-avoidance constraints for polygonal
geometry \cite{wu2025optimization}. More recently, Wang et al. employ an
LSE-aggregated composite barrier in safe reinforcement learning to obtain
a closed-form safety layer for multiple constraints
\cite{wang2025multi}. Our formulation instead employs log-mean-exp (LME), a
normalized form of LSE that preserves its smooth-max behavior while reducing
dependence on the number of aggregated terms. We use LME to aggregate per-step
safety costs into a single differentiable trajectory-level constraint for
primal--dual policy optimization. Combined with our bounded clearance-based cost, this yields a differentiable
constraint representation tailored to primal--dual policy learning for
obstacle avoidance.

\section{PROBLEM STATEMENT}

\subsection{Problem Formulation}

We consider the task of velocity tracking toward a target while avoiding
obstacles. To isolate the incorporation of safety requirements into policy
optimization from the separate problem of state estimation, we formulate the
task with the vehicle state, target position, and scene geometry available to
the policy during training and deployment. These quantities may be obtained
through external or onboard sensing and estimation systems.

\paragraph{System and Policy.}
The vehicle state is described by its position $p_t \in \mathbb{R}^3$,
velocity $v_t \in \mathbb{R}^3$, unit quaternion $q_t \in \mathbb{R}^4$,
and body rates $\omega_t \in \mathbb{R}^3$, collected as
\[
x_t = (p_t,v_t,q_t,\omega_t), \qquad \|q_t\|_2 = 1.
\]
Given a target position $p_\star\in\mathbb{R}^3$ and scene geometry
$\mathcal{S}$, the policy $\pi_\theta$ maps the current state and task
information to a collective thrust and body-rate command,
\[
u_t
=
\pi_\theta(x_t,p_\star,\mathcal{S})
=
(c_t,\omega_t^{\mathrm{cmd}})
\in\mathcal{U},
\]
where $c_t\geq0$ is the mass-normalized collective thrust and
$\omega_t^{\mathrm{cmd}}\in\mathbb{R}^3$ is the commanded body rate.

We model the vehicle as a rigid body with dynamics
\begin{equation}
\dot{p}=v,\qquad
\dot{v}=g e_3+c_t R e_3,\qquad
\dot{R}=R[\omega_t^{\mathrm{cmd}}]_\times,
\label{eq:dynamics}
\end{equation}
where $e_3=[0,0,1]^\top$ and $[\cdot]_\times$ denotes the skew-symmetric
matrix. We model the quadrotor as a rigid body and propagate its state using forward Euler integration with time step $\Delta t$. We assume that the commanded body rate is achieved over the subsequent simulation step. 

In our setting, we assume the scene geometry
$\mathcal{S}$ — the positions and radii of all obstacles — is known a priori.

\paragraph{Primal-Dual Optimization}
For a fixed initial state, target position, and scene, the policy induces an
$T$-step rollout. Sampling these quantities from a distribution $\rho$, we
define the task objective
\begin{equation}
J(\theta)
=
\mathbb{E}_{\rho}
\left[
\sum_{t=0}^{T-1}
\ell_t(x_t,u_t)
\right],
\label{eq:task_objective}
\end{equation}
where $\ell_t$ is the per-step cost.

Rather than treating obstacle violations as a weighted task penalty, we formulate safety as an explicit constraint.
\begin{equation}
\min_{\theta}\;J(\theta)
\qquad
\text{s.t.}
\qquad
C(\theta)\leq\alpha,
\label{eq:constrained_problem}
\end{equation}
where $C(\theta)$ denotes the safety cost induced by the resulting rollouts
and $\alpha$ is the maximum allowable cost. The corresponding Lagrangian is
\begin{equation}
\mathcal{L}(\theta,\lambda)
=
J(\theta)
+
\lambda\left(C(\theta)-\alpha\right),
\qquad
\lambda\geq0.
\label{eq:lagrangian}
\end{equation}
We optimize the policy parameters and the dual variable using
\begin{align}
\theta
&\leftarrow
\theta-\eta^\theta_k
\nabla_\theta\mathcal{L}(\theta,\lambda),
\label{eq:primal_update}
\\
\lambda
&\leftarrow
\left[
\lambda+\eta^{\lambda}_k
\left(C(\theta)-\alpha\right)
\right]_+,
\label{eq:dual_update}
\end{align}

where $[\cdot]_+$ ensures the dual variable remains nonnegative. $\eta^\theta_k$ and $\eta^\lambda_k$ are the respective step size.

\section{PRIMAL DUAL UPDATES WITH DIFFERENTIABLE SIMULATOR}

\subsection{Backpropagation Through  (BPTT)}

The differentiable rollout allows the gradient of the Lagrangian, Eqn. \ref{eq:lagrangian} to be computed by backpropagating through the policy and vehicle dynamics
over the rollout. In particular, the gradient is obtained by applying the
chain rule to the dependence of the Lagrangian on the sequence of states and
actions,
\begin{equation}
\nabla_\theta \mathcal{L}
=
\sum_{t=0}^{T-1}
\left(
\frac{\partial \mathcal{L}}{\partial x_t}
\frac{\partial x_t}{\partial\theta}
+
\frac{\partial \mathcal{L}}{\partial u_t}
\frac{\partial u_t}{\partial\theta}
\right),
\label{eq:bptt}
\end{equation}
where the state sensitivities $\partial x_t/\partial\theta$ are recursively
propagated through the dynamics (Eqn. \ref{eq:dynamics}) and the policy
$u_t=\pi_\theta(x_t,p_{*},\mathcal{S})$. 

\subsection{Task Objective Loss Function}
\label{sec:loss}

The task objective $J(\theta)$ measures velocity-tracking performance and
flight smoothness, without incorporating obstacle-related terms. Safety is
instead enforced separately through the constraint $C(\theta)$ introduced in
Section~\ref{sec:constraint_loss}. The reference velocity is generated geometrically
from the target position. Let
$e_t=p_\star-p_t$ denote the position error and
$\rho_t=\lVert e_t^{xy}\rVert_2$ its horizontal magnitude. We define the
horizontal reference velocity as
\begin{equation}
v_t^{\star,xy}
=
\frac{e_t^{xy}}{\rho_t+\epsilon}
v_{\max}
\operatorname{clip}\!\left(\frac{\rho_t}{\rho_0},0,1\right)
\min\!\left(\frac{t\Delta t}{T_r},1\right),
\label{eq:vref_xy}
\end{equation}
where $v_{\max}$ limits the commanded speed, $\rho_0$ reduces the reference
speed near the target, and $T_r$ smoothly ramps the command from rest.
Altitude is controlled independently using
\begin{equation}
v_t^{\star,z}
=
\operatorname{clip}
\left(k_{\mathrm{alt}}e_t^z,-v_{\max}^z,v_{\max}^z\right).
\label{eq:vref_z}
\end{equation}

The task objective combines velocity tracking, body-rate regularization, and
command smoothness respectively. Specifically, the per-step costs are
\begin{equation}
    \ell_t^{\mathrm{v}} =
w_v\left\lVert v_{t+1}-v_t^\star\right\rVert_2^2
\end{equation}
\begin{equation}
\qquad
\ell_t^{\mathrm{r}}
=
w_r\left\lVert\omega_{t+1}\right\rVert_2^2,
\qquad
\ell_t^{\mathrm{a}}
=
w_a\left\lVert a_t-a_{t-1}\right\rVert_2^2,
\label{eq:per_step}
\end{equation}
where $v_t^\star=(v_t^{\star,xy},v_t^{\star,z})$. The reference is computed
from the pre-step state, while velocity and body rates are evaluated after the
step, so each term measures the outcome of the corresponding action. The
resulting task objective is
\begin{equation}
J(\theta)
=
\frac{1}{BT}
\sum_{i=1}^{B}
\sum_{t=0}^{T-1}
\left[
\ell_t^{\mathrm{v},i}
+
\ell_t^{\mathrm{r},i}
+
\ell_t^{\mathrm{a},i}
\right]
\label{eq:task_loss}
\end{equation}
where $B$ is the number of parallel environments. The three terms respectively
encourage accurate velocity tracking, moderate body rates, and smooth control commands. 

Importantly, the weights $w_v$, $w_r$, and $w_a$ only balance different aspects of flight quality. In contrast, obstacle avoidance is represented by the separate constraint $C(\theta)\leq\alpha$, whose multiplier is adjusted during optimization rather than manually weighted against the task objective.

\begin{figure*}[t]
    \centering
    \includegraphics[width=\textwidth]{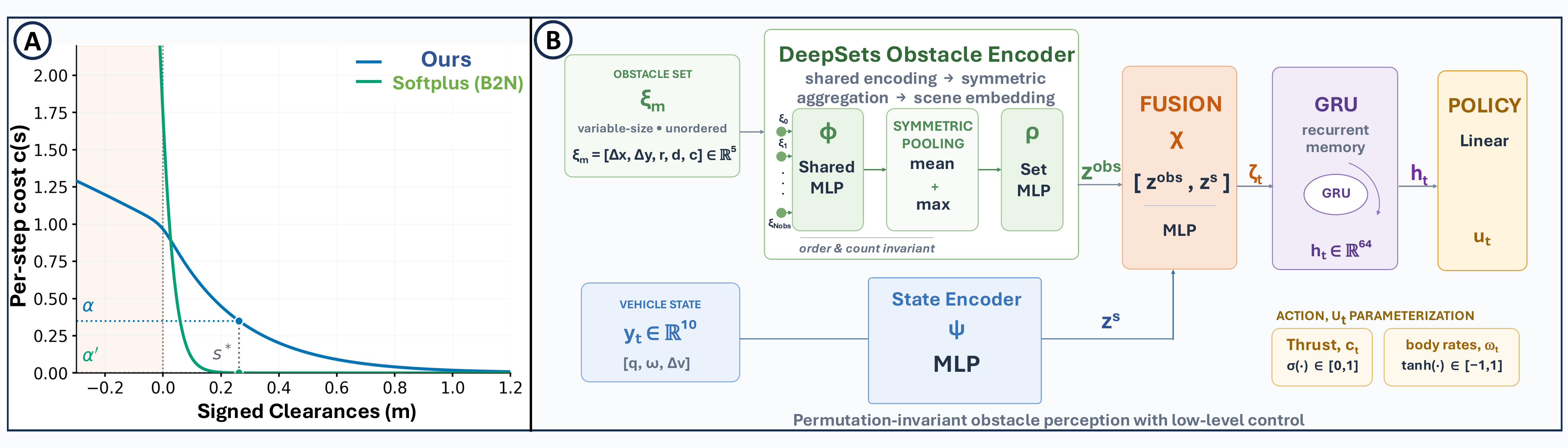}
    \caption{\textbf{A.} Comparison of our proposed bounded constraint cost with a conventional unbounded signed-distance-based cost proposed in \cite{zhang2025learning}. \textbf{B.} Our permutation-invariant policy architecture for mapping vehicle states and variable obstacle configurations to control actions.}
    \label{fig:policy_architecture}
\end{figure*}

\subsection{Constraint Loss Function}
\label{sec:constraint_loss}

Obstacle avoidance requires safety to be captured at each timestep, since even
a localized violation may result in collision. We therefore first construct a
smooth per-step safety cost from the vehicle--obstacle clearance. Let
$\mathcal{O}_m$ denote the geometry of obstacle $m$. Its signed clearance from
the vehicle at state $x_t$ is
\begin{equation}
s_{t,m} = s(x_t,\mathcal{O}_m),
\end{equation}
where $s_{t,m}>0$ indicates separation, $s_{t,m}=0$ the obstacle boundary,
and $s_{t,m}<0$ penetration. For multiple obstacles, we consider the most
critical obstacle at each timestep and define
\begin{equation}
s_t = \min_m s_{t,m}.
\label{eq:min_clearance}
\end{equation}
The resulting per-step constraint cost is then $c_t=c(s_t)$.

The choice of $c(\cdot)$ directly affects the gradients available during
primal--dual training. We seek a cost that provides informative gradients as
the vehicle approaches the safety boundary while remaining bounded under
severe penetration. We define
\begin{equation}
c(s)
=
\left[
1+(c_{\max}-1)
\tanh\!\left(K\,\mathrm{S}(-s;\kappa_s)\right)
\right]
\exp\!\left(-\beta\,\mathrm{S}(s;\kappa_s)\right),
\label{eq:cost_fn}
\end{equation}

where
\begin{equation}
\mathrm{S}(x;\kappa_s)
=
\frac{1}{\kappa_s}
\log\left(1+\exp(\kappa_s x)\right)
\end{equation}
is a smooth approximation of $\max(0,x)$. Here, $\beta$ controls the decay
of the cost with increasing clearance, $K$ controls the saturation rate under
obstacle penetration, $\kappa_s$ controls the smoothness of the approximation,
and $c_{\max}$ bounds the cost under severe penetration. Consequently, the
cost provides a continuous safety signal before collision while preventing
large violations from dominating the optimization gradient. The resulting constraint profile is illustrated in Fig.~\ref{fig:policy_architecture}A.

This produces a sequence of per-step safety costs
$\{c_t^i\}_{t=0}^{T-1}$ for each rollout $i$. Rather than using a cumulative constraint, we apply log-mean-exp (LME), a
smooth and differentiable approximation of the maximum, over time and across
the batch to emphasize high-cost events while obtaining a single scalar
constraint for the primal--dual update:
\begin{align}
C^i
&=
\frac{1}{\kappa_T}
\log
\sum_{t=0}^{T-1}
\exp\!\left(\kappa_T c_t^i\right)
-
\frac{\log T}{\kappa_T},
\label{eq:lse_time}
\\
C(\theta)
&=
\frac{1}{\kappa_E}
\log
\sum_{i=1}^{B}
\exp\!\left(\kappa_E C^i\right)
-
\frac{\log B}{\kappa_E}.
\label{eq:lse_env}
\end{align}
The normalization terms prevent the aggregate
from scaling with the trajectory length $T$ or batch size $B$.
The temperatures $\kappa_T$ and $\kappa_E$ control the degree of concentration,
interpolating from the mean toward the maximum as they increase.

\subsection{Policy Architecture}

The policy processes the vehicle state and obstacle configuration using separate
encoders. The vehicle observation is
\begin{equation}
y_t =
\left[
q_t,\;
\omega_t,\;
v_t^\star-v_t
\right]
\in\mathbb{R}^{10},
\label{eq:obs_state}
\end{equation}
where $q_t$ denotes the attitude quaternion. Absolute position and
target location are omitted, making the observation independent of the global
task location. Each obstacle is represented relative to the vehicle as
\begin{equation}
\xi_m =
\frac{1}{D}
\left[
r_m,\;
R_m,\;
\rho_m,\;
s_m
\right]
\in\mathbb{R}^{5},
\qquad
\label{eq:obs_set}
\end{equation}
where $r_m$ is the horizontal displacement to the $m^{th}$ obstacle centre,
$R_m$ its radius, $\rho_m=\|r_m\|_2$, and
$s_m=\rho_m-R_m$ the corresponding clearance. Padding allows scenes with varying numbers of obstacles to be processed in fixed-size batches.

To make the obstacle representation invariant to obstacle ordering, we use a permutation-invariant Deep Sets encoder \cite{zaheer2017deep}. A shared network
$\phi:\mathbb{R}^{5}\rightarrow\mathbb{R}^{d_o}$ maps each obstacle to a
feature, which is aggregated using mean and max pooling:
\begin{align}
\bar{z} &=
\frac{1}{n}\sum_m \phi(\xi_m),
&
\hat{z} &=
\max_{m:\mu_m=1}\phi(\xi_m),
\\
z^{\mathrm{obs}}
&=
\varrho\left(
\left[
\bar{z},\hat{z},\log(1+n)
\right]
\right),
\label{eq:pooling}
\end{align}
where $n=\max(\sum_m1)$ and $\varrho$ is an MLP. The symmetric pooling
ensures that the representation does not depend on obstacle ordering, while
the count feature retains information about the number of visible obstacles.

The vehicle and obstacle embeddings are then fused and processed by a recurrent
policy. Specifically, the state observation is embedded by
$\psi:\mathbb{R}^{10}\rightarrow\mathbb{R}^{d_s}$ and concatenated with
$z_t^{\mathrm{obs}}$, after which an MLP $\chi$ produces the latent feature
$\zeta_t$. A GRU maintains the recurrent state
\begin{equation}
h_t =
\mathrm{GRU}(\zeta_t,h_{t-1}),
\qquad
h_{-1}=\mathbf{0},
\end{equation}
which provides temporal memory for more effective sim-to-real deployment. This is followed by a linear head,  $\mathcal{H}$, that outputs the four-dimensional CTBR command. 
\begin{equation}
u_t
=
\left[
c_t,\,
\omega_t^{\mathrm{cmd}}
\right]
=
\mathcal{H}(h_t)
\in\mathcal{U},
\label{eq:policy_output}
\end{equation}
The overall policy architecture is shown in Fig. \ref{fig:policy_architecture}B.


\subsection{Algorithm and Feasibility Convergence}
\label{ssec:proof}

Summarizing the above methodology, we have proposed a primal dual algorithm for optimizing a novel constrained finite-horizon decision-making problem under BPTT. Next, we address the challenge of establishing the convergence and feasibility of the proposed primal-dual algorithm. Unlike classical approaches for CMDPs that usually conduct approximate policy gradient updates, our algorithm relies on deterministic backpropagation through time to compute the gradients of a cumulative reward objective and a general constraint function w.r.t. the deterministic policy parameters. The following theorem show that our algorithm is guaranteed to converge to policy that satisfies the constraints, ensuring feasibility while optimizing the objective function for sim-to-real deployment.

\begin{restatable}{theorem}{informalthm}
\label{thm:convergence_feasibility}
    Let the update rule from Eq.~(\ref{eq:primal_update}) and (\ref{eq:dual_update}) at iteration step $k\in\mathbb N_0$ be given by
    \begin{align}
        & \theta_{k+1} = \theta_k - \eta_k^{\theta} \left( \nabla_\theta J(\theta_k) + \lambda_k\nabla_\theta C(\theta_k) \right) \\
        & \lambda_{k+1} = \left[ \lambda_k + \eta_k^{\lambda} \left(C(\theta_k) - \alpha\right) \right]_+,
    \end{align}
    where $\eta_k^{\theta}$ and $\eta_k^{\lambda}$ follow the Robbins-Monro stepsizes, i.e., $\sum_{k=0}^{\infty} \eta_k^{\theta}=\infty$, $\sum_{k=0}^{\infty} \eta_k^{\lambda}=\infty$, $\bar\rho_n \coloneq \sup_{k\ge n} \nicefrac{\eta_k^{\lambda}}{\eta_k^{\theta}} \rightarrow0$ and $\eta_k^{\theta} \left( 1+\sum_{j=0}^{k-1}\eta_j^{\lambda} \right) \rightarrow0$.
    Assume that there exists a compact convex set $\Theta \subset\mathbb R^p$ such that $\forall k$, $\theta_k\in\Theta$, and there exist constants $\delta>0$ and $\nu>0$ such that $\tilde{C}(\theta)\ge-\delta \;\; \Rightarrow \;\; \|\nabla C(\theta)\|\ge\nu$, then during optimization $\lambda_k$ remains bounded $\lambda_k\in[0,\Lambda_0]$, and we can obtain an asymptotically feasible policy such that $\forall\varepsilon>0$, $\forall n,$ $\exists k\ge n: C(\theta_k)\le\alpha+\varepsilon$.
\end{restatable}

\begin{proof}[Proof sketch]
Let $\tilde C=C-\alpha$. A Taylor expansion shows that, near the constraint boundary, the primal update strictly decreases $\tilde C$ whenever $\lambda_k$ is large.
Timescale separation makes dual growth during such excursions vanish; then negative dual update keeps $\lambda_k$ bounded.
Moreover, $\eta_k^\lambda\tilde C(\theta_k)\le\lambda_{k+1}-\lambda_k$.
Summing and using $\sum_k\eta_k^\lambda=\infty$ gives $\liminf_k\tilde C(\theta_k)\le0$, proving asymptotic feasibility.\footnote{The complete proof is given at the anonymous site: https://anonymous.4open.science/r/convergence-proof/appendix.pdf}
\end{proof}

\section{EXPERIMENTS}

Our experiments address three questions, progressing from comparisons with existing methods to ablations of the key components of our formulation.

\begin{enumerate}
    \item \textbf{How does our method compare with existing policy-learning approaches?}
    We compare against penalty-based, model-free, and trust-region constrained methods, as well as fixed-$\lambda$ variants of our formulation, evaluating task performance, safety, and their trade-off (Section~\ref{sec:exp_baselines}).

    \item \textbf{How does the bounded constraint cost affect optimization?}
    We compare the proposed bounded cost with unbounded and hinge-based alternatives, evaluating constraint satisfaction and primal--dual training stability. (Section. \ref{sec:ablation})

    \item \textbf{What does LME aggregation provide over cumulative constraints?}
    We assess whether LME better captures critical per-step safety while maintaining
    a consistent constraint interpretation across rollout horizons.
    (Section~\ref{sec:aggregation})
\end{enumerate}
 
\subsection{Training Details}
\label{sec:training_details}

Policies are trained in a fixed-size room with randomized initial states and obstacle configurations. Each environment contains four effectively infinite vertical cylinders of radius $0.7$~m, with positions randomized subject to a minimum surface-to-surface separation of $1.4$~m. Initial and goal positions are sampled from collision-free regions with at least $1$~m obstacle clearance, while the initial horizontal velocity has a uniformly random direction and magnitude in $[0,0.2]$~m/s. Environments are resampled every epoch. Collision and clearance are computed with respect to the vehicle's center of mass. Training parameters are summarized in Table~\ref{tab:training_params}. While the analysis in Sec. \ref{ssec:proof} considers asymptotic primal--dual updates, we select practical optimization parameters to achieve stable convergence within a finite training budget.

\subsection{Baseline Comparisons}
\label{sec:exp_baselines}

We compare against baselines spanning different policy-learning and constraint
formulations. B2N~\cite{zhang2025learning} represents differentiable
quadrotor navigation with penalty-based obstacle avoidance, while
CGPO~\cite{dai2024safe} provides a constrained differentiable-simulation
baseline using adaptive trust-region updates. We additionally evaluate
penalty-based PPO to contrast model-free and analytic-gradient learning. Finally, evaluate fixed-$\lambda$ variants of our formulation,
where $\lambda$ is held constant during training, to isolate the benefit of
adaptive dual updates.

\begin{table}[t]
\centering
\scriptsize
\setlength{\tabcolsep}{2.5pt}
\caption{Training parameters.}
\label{tab:training_params}
\begin{tabular}{lll|lll}
\hline
\multicolumn{6}{c}{\textit{Task and training}} \\
Symbol & Description & Value &
Symbol & Description & Value \\
\hline
$B$ & Parallel envs. & $256$ &
$T$ & Horizon & $300$ steps \\
$\Delta t$ & Time step & $0.03$ s &
$N_{\mathrm{obs}}$ & Obstacles & $4$ \\
$N_{\mathrm{epoch}}$ & Epochs & $20$k &
$v_{\max}$ & Horiz. vel. & $1.0$ m/s \\
$v_{\max}^z$ & Vert. vel. & $0.3$ m/s &
& & \\
\hline

\multicolumn{6}{c}{\textit{Constraint}} \\
$\alpha$ & Threshold & $0.35$ &
$\beta$ & Decay & 4.0 \\
$c_{\max}$ & Cost bound & 2.0 &
$K$ & Penetration gain & 1.0 \\
$\kappa_s$ & Softplus temp. & 60.0 &
$\kappa_T$ & Temp. (time) & 20.0 \\
$\kappa_E$ & Temp. (env.) & 5.0 &
& & \\
\hline

\multicolumn{6}{c}{\textit{Loss and optimization}} \\
$w_v$ & Vel. weight & 50.0 &
$w_r$ & Rate weight & 5.0 \\
$w_a$ & Act. weight & 20.0 &
$\eta^\theta$ & Policy LR & 5e-4 \\
$\eta^\lambda$ & Dual LR & 0.05 &
$\lambda_0$ & Initial dual & 8.0 \\
\hline
\end{tabular}
\end{table}

\begin{figure}[t]
    \centering
    \includegraphics[width=0.8\columnwidth]{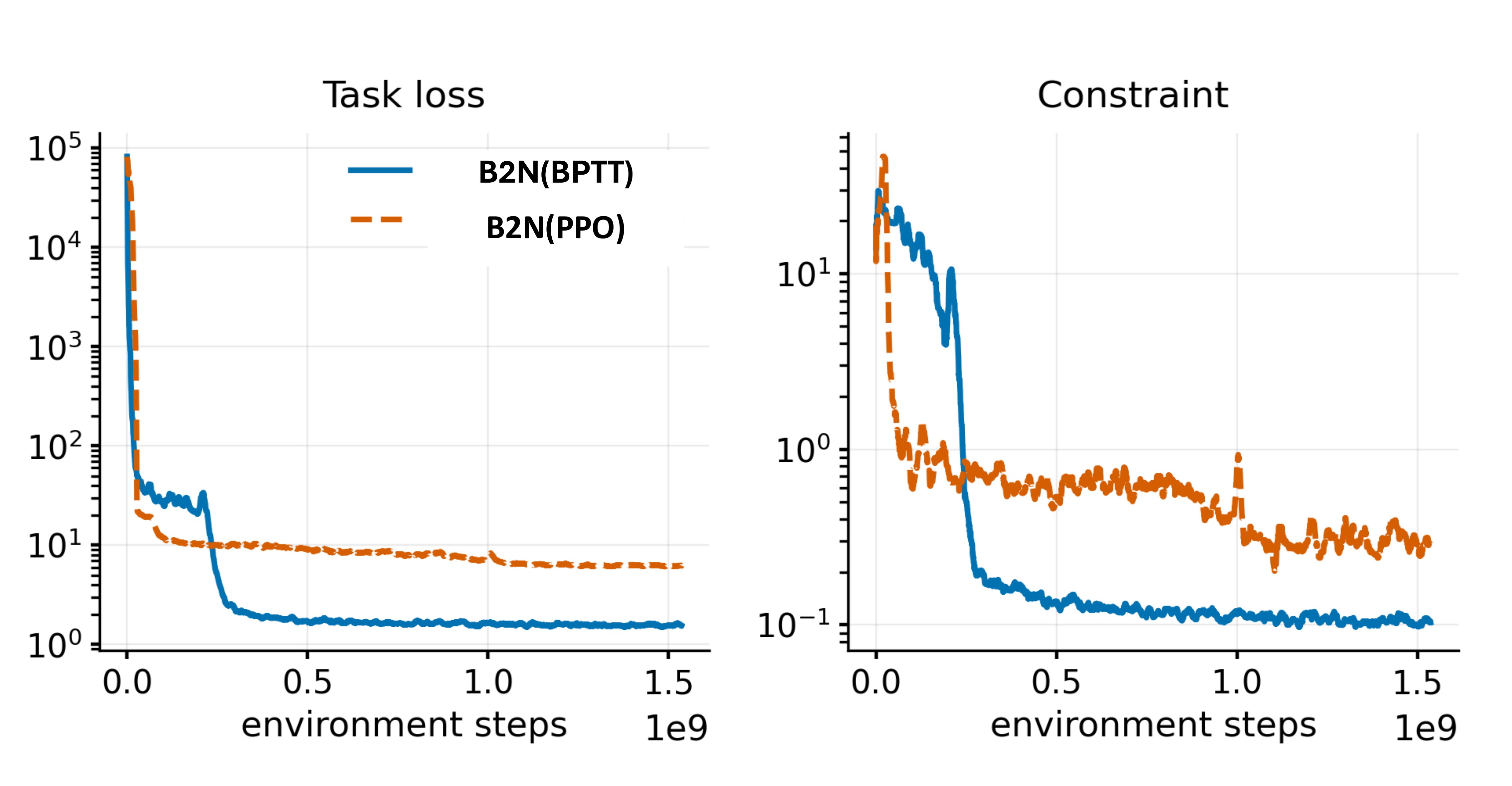}
    \caption{Training convergence of B2N (BPTT) and B2N (PPO), with B2N (BPTT)
    converging to a better-performing policy.}
    \label{fig:convergence_curve}
    \vspace{-0.5cm}
\end{figure}

All baselines are trained under the same environment settings described in Section~\ref{sec:training_details}, using the same policy inputs, task objective (Section \ref{sec:loss}), dynamics model, action space, episode horizon, and training/evaluation scene distributions. Unless required by the respective algorithm, we also keep the policy backbone. Thus, the main differences across methods arise from their treatment of collision avoidance and the corresponding optimization procedure. For B2N (BPTT), we largely follow its original cumulative collision-penalty formulation. B2N (PPO) uses the same task and collision objectives as B2N, but with the policy adapted for stochastic action outputs and an additional critic of the same backbone architecture for value estimation. Since CGPO is formulated for finite-horizon cumulative constraints, we use the same per-step constraint cost defined in Eq.~\ref{eq:cost_fn}, but aggregate it cumulatively over the trajectory rather than using the proposed LME aggregation and set its $\alpha = 5.0$ and $40$. Finally, to isolate the effect of adapting the dual variable, we compare our proposed primal--dual formulation against fixed-$\lambda$ variants spanning a range of multiplier values.

\paragraph{Evaluation Metrics}
We evaluate each policy using five metrics capturing task performance, constraint satisfaction, safety, and navigation success. \textit{Task cost} (Eqn.~\ref{eq:task_loss}) and \textit{LME} (Eqn.~\ref{eq:lse_time}) quantify task performance and aggregate constraint violation, respectively, with lower values being preferable. \textit{Clear. p5} is the 5th percentile of the minimum obstacle clearance across the method's rollouts, where positive values indicate collision-free clearance and negative values indicate obstacle penetration. \textit{Coll.} reports the percentage of trials in which the quadrotor's center of mass enters an obstacle. Finally, \textit{Success} reports the percentage of trials in which the goal is reached within a 0.2~m radius.

\paragraph{Results}
We evaluate all methods on the same 10,240 scenarios, comprising 64 randomized
obstacle layouts and 160 obstructed start--goal pairs per layout, with each
rollout lasting 600 time steps. Table~\ref{tab:baselines} shows a clear
safety--performance trade-off for fixed-$\lambda$ training: increasing
$\lambda$ from $4$ to $32$ improves the 5th-percentile clearance from
$-0.662$~m to $+0.081$~m, but reduces success from $100\%$ to $19.4\%$
($67.9\%$ at $\lambda=16$). In contrast, our method achieves $100\%$ success
with positive clearance ($+0.018$~m) and an LME of $0.571 \pm 0.133$.
Notably, the learned dual converges to $\lambda=2.29$, yet substantially
improves safety over the fixed $\lambda=4$ policy while preserving task
completion, highlighting the benefit of adapting the dual variable rather
than manually selecting a fixed penalty.


B2N (BPTT) achieves the lowest LME ($0.019 \pm 0.006$) and highest
5th-percentile clearance ($+0.665$~m), with $100\%$ success. B2N (PPO)
also performs competitively, but achieves lower success ($90.5\%$) and
clearance ($+0.504$~m). During training, BPTT also converges to lower task
and B2N collision costs than PPO (See Fig. \ref{fig:convergence_curve}). Together, these results motivate our use
of analytic BPTT gradients, rather than model-free PPO, as the basis for the proposed primal--dual formulation. Among methods achieving $100\%$ success with positive
5th-percentile clearance, our method operates closer to the safety
boundary while reducing task cost by $31.2\%$ relative to B2N
(BPTT), demonstrating a less conservative safety--performance
trade-off.



CGPO exhibits strong sensitivity to its cumulative constraint limit $\alpha$.
At $\alpha=5$, it maintains positive 5th-percentile clearance ($+0.042$~m)
but achieves only $10.3\%$ success. Relaxing the limit to $\alpha=40$
increases success to $97.7\%$, but results in a $92.75\%$ collision rate and
a 5th-percentile clearance of $-0.664$~m. These results illustrate the
difficulty of selecting a cumulative constraint limit: $\alpha$ bounds total
constraint exposure over the trajectory and therefore does not map directly
to a desired critical obstacle clearance. Our LME formulation instead
emphasizes the largest per-step costs, providing a more direct relationship
between $\alpha$ and critical clearance through the monotonic constraint
function $c(s)$. With the selected threshold, our policy achieves $100\%$
success, a $0.44\%$ collision rate, and positive 5th-percentile clearance
($+0.018$~m).

We further compare our approach with B2N (BPTT) in a more challenging out-of-distribution setting, increasing the number of obstacles from four to eight and reducing the minimum inter-obstacle separation to $0.5$~m. Across 2,112 scenarios, our method maintains a $99.0\%$ success rate, whereas B2N (BPTT) drops to $58.3\%$ and exhibits a larger 5th-percentile obstacle incursion ($0.383$~m vs.\ $0.192$~m). Representative failures are shown in Fig.~\ref{fig:four_panel}B. These results suggest that B2N's fixed penalty weighting may require retuning under distribution shift, whereas our method generalizes more effectively to denser environments. Residual safety violations can be further mitigated through obstacle inflation without retraining.

\begin{figure*}[t]
    \centering
    \includegraphics[width=0.24\textwidth]{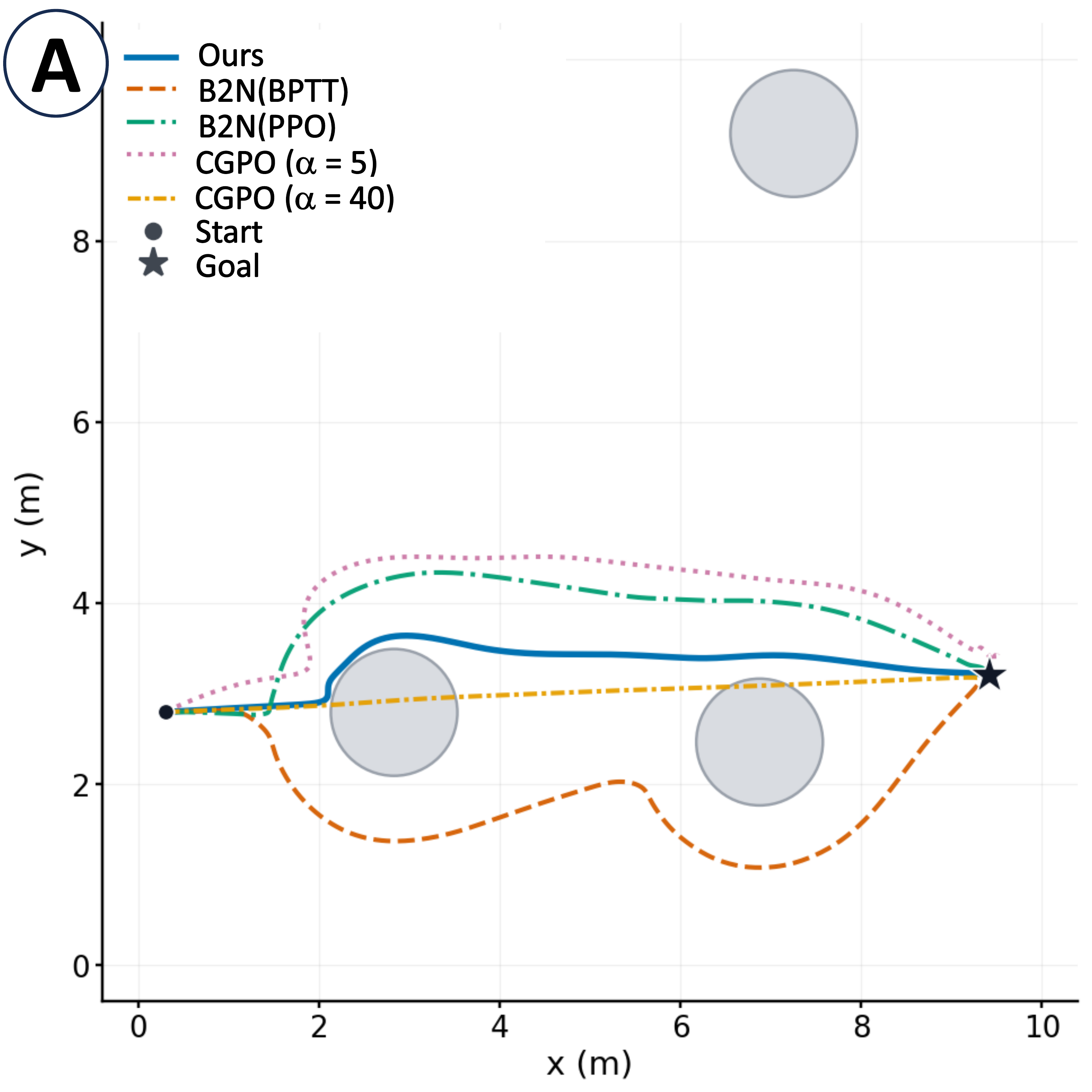}
    \hfill
    \includegraphics[width=0.24\textwidth]{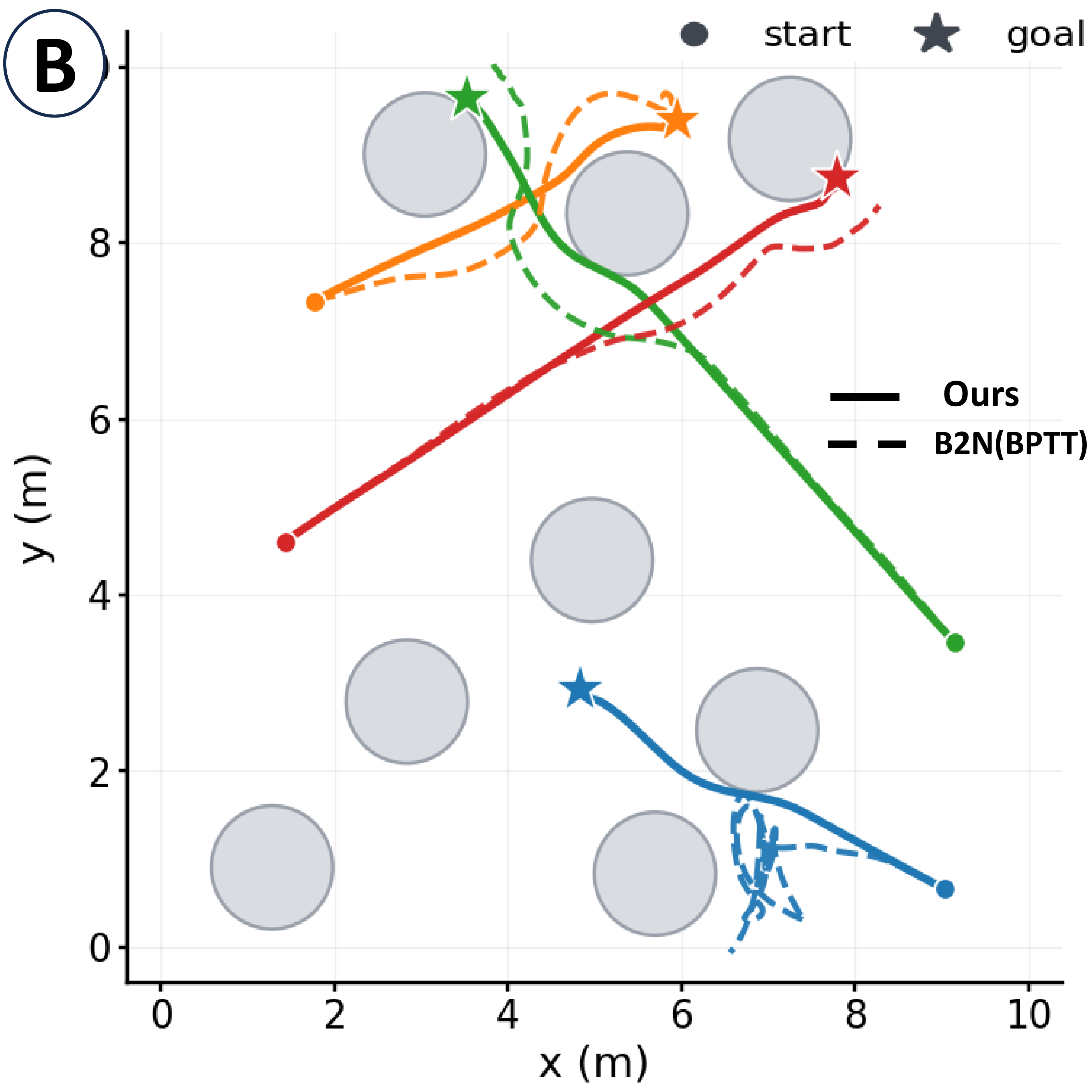}
    \hfill
    \includegraphics[width=0.24\textwidth]{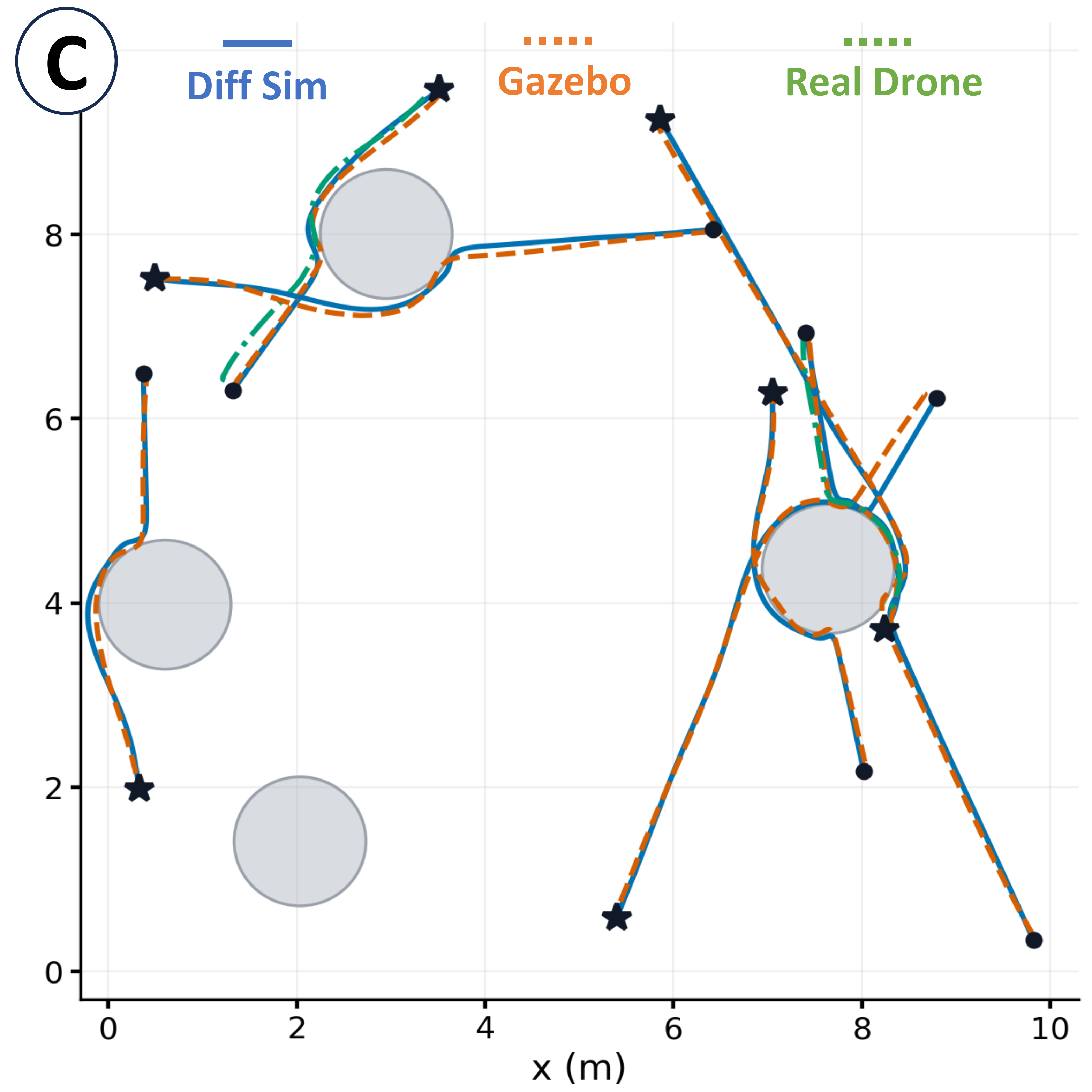}
    \hfill
    \includegraphics[width=0.24\textwidth]{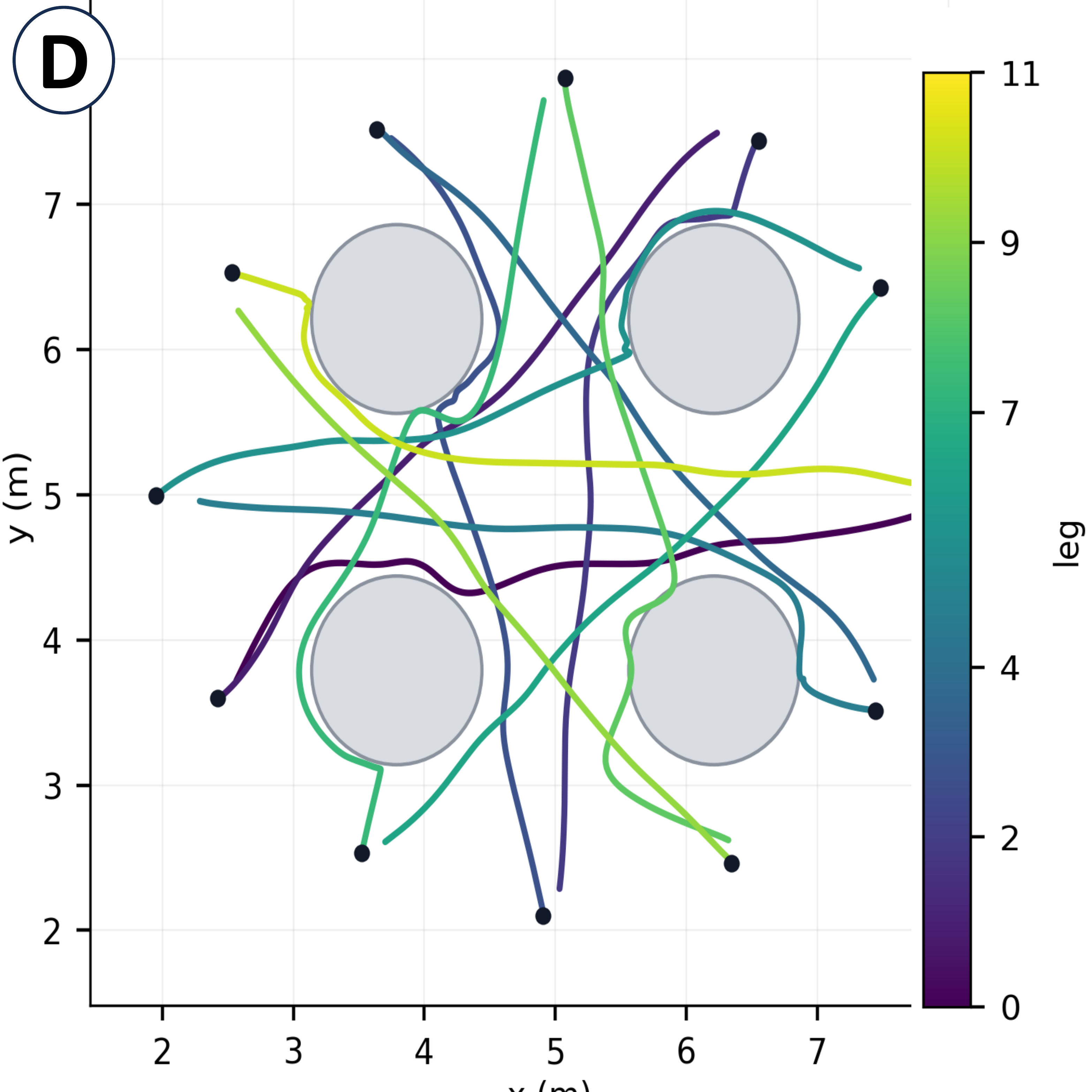}

    \caption{
    \textbf{A.} Representative trajectories of the baseline policies and our proposed method.
    \textbf{B.} Representative OOD trajectories comparing our method with B2N(BPTT). B2N fails near obstacle-adjacent goals (green/red) and in the narrow passage (blue), while ours succeeds in all cases.
    \textbf{C.} Consistent trajectories across differentiable simulation, Gazebo,
    and real-world deployment demonstrate zero-shot sim-to-real transfer.
    \textbf{D.} Real-world quadrotor trajectory over a 12-leg continuous circuit,
    with waypoints distributed on a $3$~m-radius circle centered on the obstacle
    configuration.
    }
    \label{fig:four_panel}
\end{figure*}

\begin{table}[t]
\centering
\scriptsize
\setlength{\tabcolsep}{0.9pt}
\caption{Held-out comparison over 10,240 obstructed scenarios
(mean $\pm$ std.). Shaded rows denote methods achieving 100\%
success with positive 5th-percentile clearance. Within this subset,
bold indicates the lower task cost.}
\label{tab:baselines}

\begin{tabular}{lccccc}
\toprule
Method
& Task $\downarrow$
& LME $\downarrow$
& Coll. (\%) $\downarrow$
& Clear. p5 (m) $\uparrow$
& Success $\uparrow$ \\
\midrule

B2N (PPO)
& $3.48{\pm}2.73$
& $0.027{\pm}0.054$
& $0.18$
& $+0.504$
& $90.5\%$ \\

\rowcolor{gray!12}
B2N (BPTT)
& $1.70{\pm}0.77$
& $0.019{\pm}0.006$
& $0.00$
& $+0.665$
& $100.0\%$ \\

CGPO ($\alpha / 5$)
& $6.40{\pm}5.56$
& $0.129{\pm}0.214$
& $4.12$
& $+0.042$
& $10.3\%$ \\

CGPO ($\alpha/ 40$)
& $0.52{\pm}0.11$
& $1.098{\pm}0.217$
& $92.75$
& $-0.664$
& $97.7\%$ \\

$\lambda=4$
& $0.34{\pm}0.00$
& $1.122{\pm}0.171$
& $99.23$
& $-0.662$
& $100.0\%$ \\

$\lambda=16$
& $6.76{\pm}4.99$
& $0.450{\pm}0.202$
& $8.66$
& $-0.044$
& $67.9\%$ \\

$\lambda=32$
& $12.05{\pm}4.10$
& $0.504{\pm}0.077$
& $0.29$
& $+0.081$
& $19.4\%$ \\

\rowcolor{gray!12}
\textbf{Ours}
& $\mathbf{1.17{\pm}0.77}$
& $0.571{\pm}0.133$
& $0.44$
& $+0.018$
& $100.0\%$ \\

\bottomrule
\end{tabular}
\end{table}

\begin{figure}[t]
    \centering

    \includegraphics[
        width=0.49\columnwidth,
        trim=0 15 0 15,
        clip
    ]{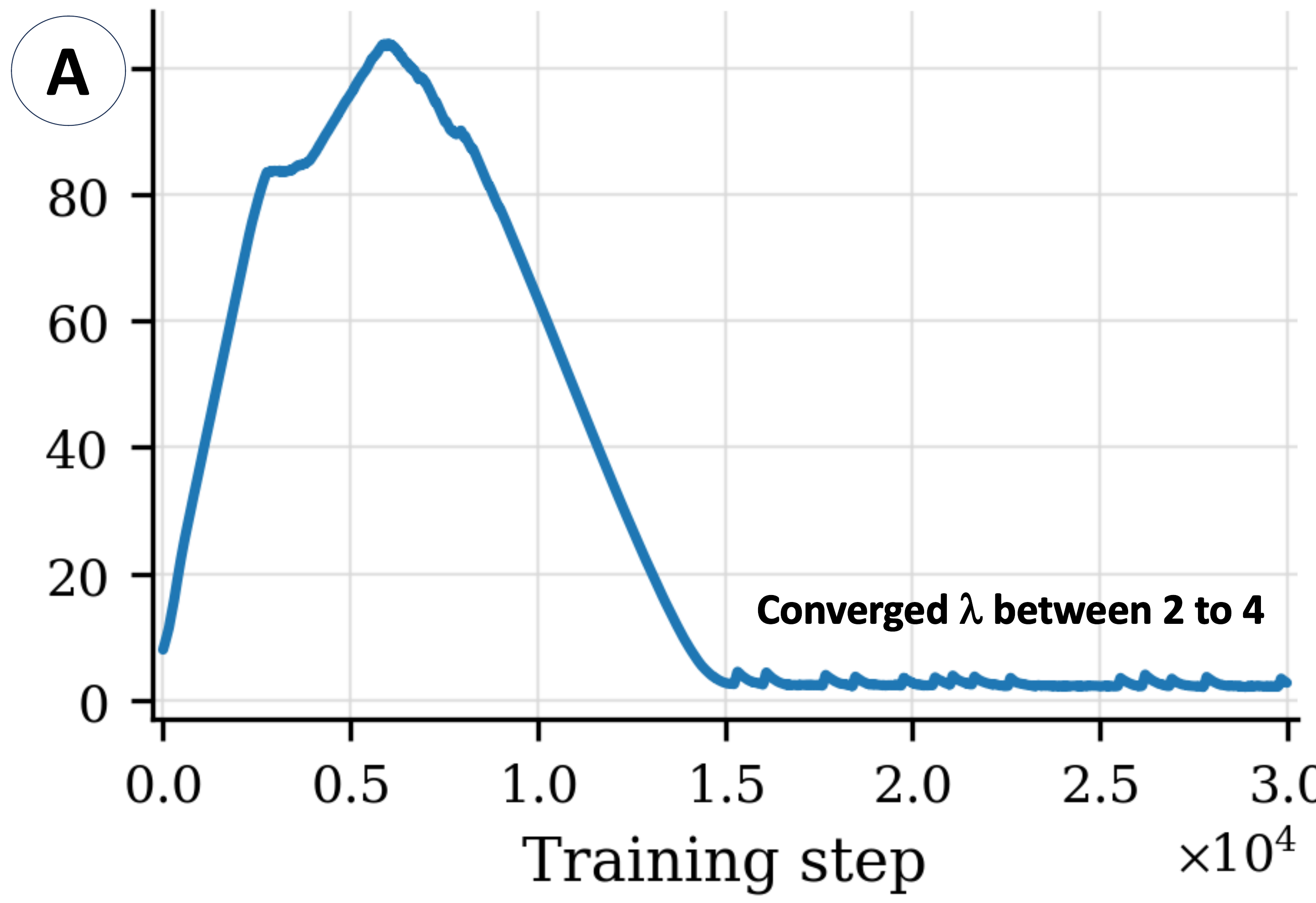}
    \hfill
    \includegraphics[
        width=0.49\columnwidth,
        trim=0 15 0 15,
        clip
    ]{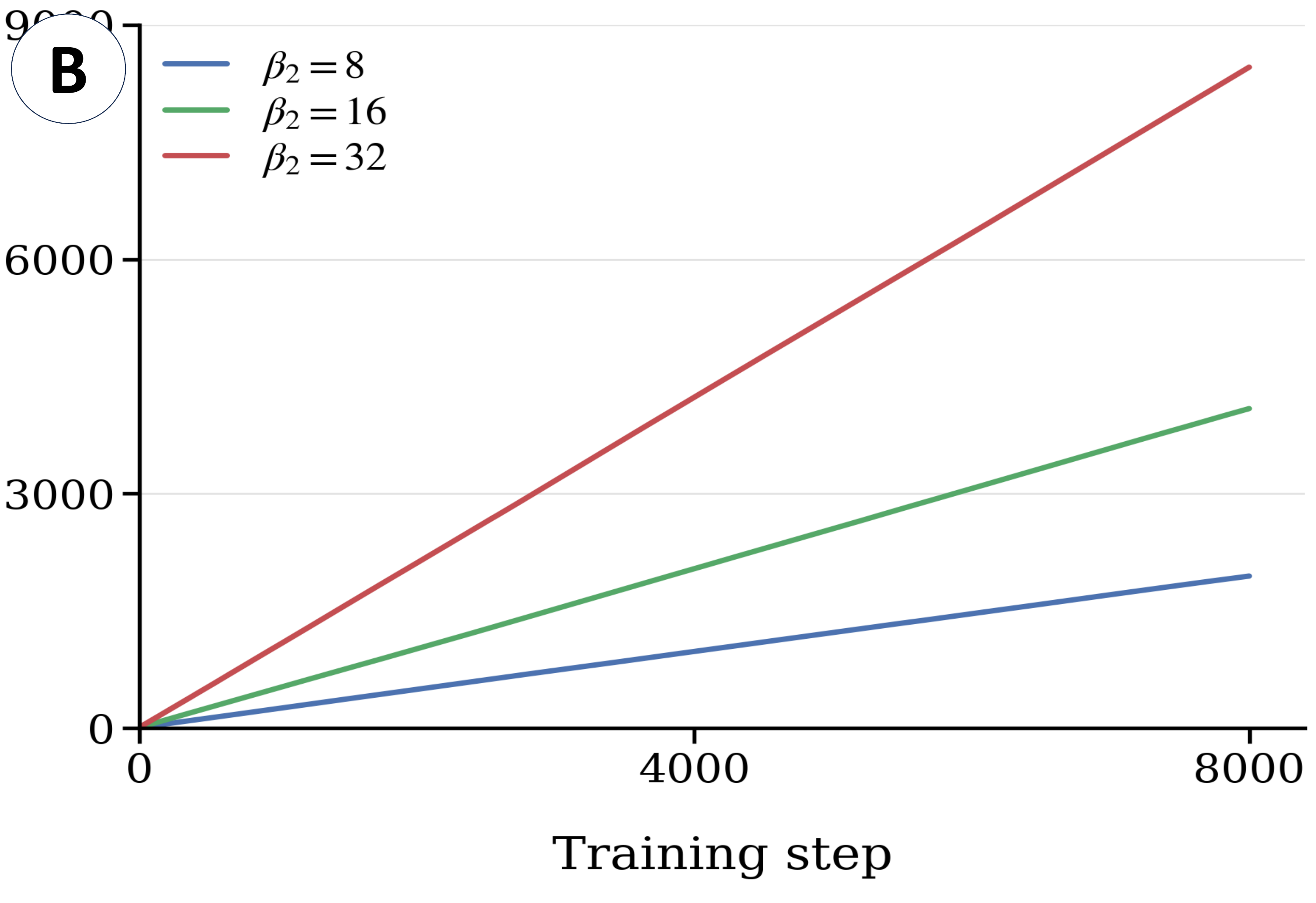}

    \vspace{2pt}

    \includegraphics[
        width=\columnwidth,
        trim=0 15 0 15,
        clip
    ]{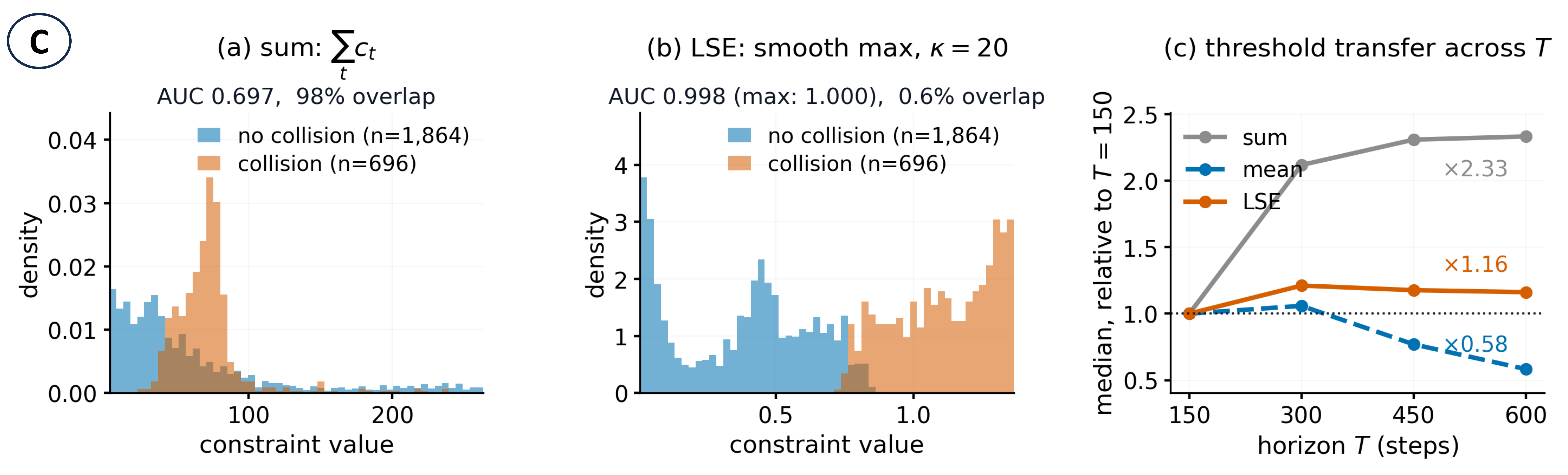}

    \caption{
    \textbf{A.} $\lambda$ evolution with our proposed bounded constraint cost.
    \textbf{B.} $\lambda$ evolution with the unbounded constraint cost used by B2N \cite{zhang2025learning} with $\beta_2$ parameter sweep.
    \textbf{C.} Analysis of constraint aggregation. LME closely matches the exact
maximum in separating colliding from collision-free trajectories and remains
consistent across rollout horizons, unlike cumulative aggregation.
    }
    \label{fig:lambda}
\end{figure}
\subsection{Ablation Study: Effect of the Bounded Constraint Cost.}
\label{sec:ablation}

We first evaluate the effect of the bounded per-step constraint cost on
primal--dual training. We replace our proposed cost
(Eq.~\ref{eq:cost_fn}) with the unbounded signed-distance-based cost used
by B2N~\cite{zhang2025learning}, while keeping the LME aggregation and
primal--dual optimization unchanged.

For a fair comparison, we align both constraint functions at the same physical
operating point. As illustrated in Fig.~\ref{fig:policy_architecture}A, our
threshold $\alpha$ corresponds to a signed clearance $s^*$; we then set the
B2N threshold $\alpha'$ to its constraint cost at the same clearance. Thus,
$\alpha$ and $\alpha'$ represent the same physical safety boundary despite
the different constraint functions. This alignment is meaningful because both
formulations aggregate per-step costs using LME, which smoothly approximates
the maximum constraint cost along the trajectory and therefore emphasizes
low-clearance states.

Fig.~\ref{fig:lambda}A and B show the evolution of the dual variable under both constraint functions. With our bounded cost, the dual variable settles to a steady-state value once the policy satisfies the constraint. Under the B2N cost it instead grows without
settling over the course of training.



This follows from the shape of the two per-step costs. The dual update increases $\lambda$ when the aggregated constraint exceeds the budget and decreases it otherwise, with the step size proportional to the violation. Under the B2N cost this signal is strongly asymmetric. The softplus grows linearly once the trajectory penetrates an obstacle, so a violation produces a large positive update, but it decays exponentially with clearance, so a satisfying trajectory produces a violation of near-zero magnitude and hence almost no downward correction. Increases therefore dominate decreases, and $\lambda$ ratchets upward across training. Because our per-step cost saturates on violation and decays gradually
with clearance, the upward and downward signals remain comparable in magnitude. The dual variable is therefore corrected in both directions at similar rates and settles to a steady-state value, yielding more stable training.

\subsection{Effect of Constraint Aggregation.}
\label{sec:aggregation}

We analyze whether LME provides a reliable surrogate for critical per-step
safety using trajectories generated by our policy and several
baselines. Using geometric collision as the ground-truth label, we compare LME, the exact maximum, and cumulative aggregation using the area under the receiver operating characteristic curve (AUROC), which measures how well each metric separates colliding from collision-free trajectories across all possible thresholds. LME achieves an AUROC of $0.998$, closely matching the exact maximum ($1.000$), while cumulative aggregation achieves only $0.697$,
indicating substantially greater overlap between the scores assigned to
colliding and collision-free trajectories. Moreover, cumulative aggregation vary substantially with trajectory length, whereas LME remains close to the maximum when increasing the horizon from $T=300$ used during training to $T=600$ during evaluation. These results indicate that LME both preserves sensitivity to localized violations and provides a more consistent constraint scale across rollout horizons.

\subsection{Zero-Shot Sim-to-Sim and Sim-to-Real Deployment}
\label{sec:sim2real}

To evaluate transfer beyond the differentiable simulator, we deploy the learned policy zero-shot in Gazebo/PX4 and on a physical quadrotor, without additional training or adaptation.

\paragraph{Sim-to-Sim Transfer.}
We first deploy the policy in Gazebo with PX4 in the loop, introducing
higher-fidelity dynamics and an independent low-level controller. Across 100
trajectories, we compare velocity-tracking cost, the LME-aggregated proposed
constraint cost (Eqs.~\ref{eq:cost_fn} and~\ref{eq:lse_time}), and obstacle
incursion (Table~\ref{tab:sim2sim}). The simulators show close agreement in
tracking, with an $13.2\%$ difference in velocity cost and a positional RMSE of
$0.369 \pm 0.234$\,m. The LME increases from $0.669$ to $0.863$, while the
mean deepest incursion increases from $0.1$ to $18.8$\,mm, likely due to
unmodeled noise and differences in low-level controller response. The
worst-case incursion remains $97.1$\,mm, suggesting that the observed
sim-to-sim safety gap can be accommodated in practice by introducing a
corresponding obstacle-inflation margin during training. Figure \ref{fig:four_panel}C also show the closely-matched trajectories.

\begin{table}[t]
  \centering
  \caption{
  Sim-to-sim comparison (mean $\pm$ std.).}
  \label{tab:sim2sim}
  \setlength{\tabcolsep}{6pt}
  \begin{tabular}{lcc}
    \toprule
    Metric & Diff Sim & Gazebo \\
    \midrule

    Velocity tracking cost, $\ell_t^{\mathrm{v}}$         & $0.068 \pm 0.031$   & $0.077 \pm 0.034$ \\
    LME constraint               & $0.669 \pm 0.114$   & $0.863 \pm 0.134$ \\
    Deepest incursion (\si{\meter})
                                 & $0.0001 \pm 0.0012$ & $0.0188 \pm 0.0217$ \\
    \bottomrule
  \end{tabular}
  \vspace{-0.8cm}
\end{table}

\paragraph{Sim-to-Real Transfer.}
We deploy the same policy zero-shot on a physical quadrotor with PX4, using
Vicon and onboard IMU measurements for state estimation. Across $10\%$ of the
sim-to-sim test cases, the policy achieves a velocity-tracking cost of $0.0958$
and LME of $0.836$, with no obstacle incursions, demonstrating collision-free
zero-shot transfer.

We further evaluate a cluttered-center course comprising four $0.71$~m-radius
cylindrical obstacles with a minimum corridor width of $1.01$~m. Twelve
waypoints on a $3$~m-radius ring are executed consecutively as a continuous
circuit (Fig.~\ref{fig:four_panel}D). The quadrotor reaches all 12 goals, with
11 legs collision-free and a maximum intrusion of $8$~cm on the remaining leg.

\section{CONCLUSIONS}

We presented a constrained differentiable policy learning formulation for
quadrotor obstacle avoidance that replaces manually weighted safety penalties
with an explicit safety constraint. Two components enable effective
primal--dual training: a bounded clearance-based constraint cost that provides
informative gradients while limiting the influence of severe violations, and
LME aggregation that emphasizes critical constraint costs within a smooth
trajectory-level constraint. Experiments demonstrate favorable
safety--performance trade-offs, generalization to denser out-of-distribution
obstacle configurations, and zero-shot transfer to Gazebo/PX4 and a physical
quadrotor. Future work will extend the framework to partially observable
settings and vision-based navigation, removing the reliance on directly
available state and obstacle information.

\addtolength{\textheight}{-12cm}   





\bibliographystyle{IEEEtran}
\bibliography{IEEECmdp}




\clearpage
\addtolength{\textheight}{12cm}
\global\vsize=\textheight
\global\pagegoal=\textheight
\makeatletter
\global\@colroom=\textheight
\global\@colht=\textheight
\makeatother
\appendix
\section{Convergence Analysis}
\label{appsec:theory_proof}

In this file, we provide a detailed proof of the convergence of our proposed method. We first introduce the necessary notation and formally define the problem.
Then, we restate the main theorem and provide a rigorous proof.

\subsection{Notation and Problem Definition}

Consider the constrained optimization problem described in the main text, for ease of presentation we consider the case in a single environment \(B=1\) with a single obstacle \(M=1\) and the start and goal states fixed\footnote{This formulation can be straightforwardly generalized to multiple environments and obstacles. The extension to sampled start and goal states follows by introducing standard stochastic-approximation arguments, under the usual regularity conditions of bounded martingale-difference noise.}.
The objective function in our problem is
\begin{equation}
    J(\theta) = \frac1T \sum_{t=0}^{T-1} \ell_t(x_t, u_t),
\end{equation}
and the constraint function is then
\begin{equation}
    C(\theta) = \frac{1}{\kappa_T} \log \left( \frac1T \sum_{t=0}^{T-1} e^{\kappa_T c_t(\theta)} \right),
\end{equation}
where the cost function $c(s) = [1 + (c_{\max}-1) \tanh(K S(-s))] e^{-\beta S(s)}$ and $S(s) = \frac1{\kappa_s} \log(1 + e^{\kappa_s s})$. With parameters $c_{\max}>1$, $K>0$, $\beta>0$ and $\kappa_s>0$ chosen in the paper, we have $0<c(s)<c_{\max}$ and $c'(s)<0$. The cost $c_t(\theta):=c(s_t(\theta))$ at time step \(t\in\{0,\ldots,T-1\}\) with \(T\in\mathbb N\) is derived from the signed clearance \(s_t(\theta)=s(x_t, \mathcal{O})\) induced by the effect of policy parameter \(\theta\in\mathbb R^p\) on the vehicle state $x_t$.

The original constrained optimization problem can be equivalently expressed as
\begin{equation}
    \min_\theta J(\theta) \;\;\; \text{s.t.}\;\;\; \tilde{C}(\theta)\le0,
\end{equation}
with $\tilde{C}(\theta) := C(\theta)-\alpha$. The Lagrangian is $L(\theta,\lambda) = J(\theta)+\lambda \tilde{C}(\theta)$ with $\lambda\ge0$. The primal-dual recursion is
\begin{align}
    & \theta_{k+1} = \theta_k - \eta_k^{\theta} \left[ \nabla_\theta J(\theta_k) + \lambda_k\nabla_\theta C(\theta_k) \right], \label{eq_app:update_primal} \\
    & \lambda_{k+1} = [ \lambda_k + \eta_k^{\lambda}\tilde{C}(\theta_k) ]_+. \label{eq_app:update_dual}
\end{align}
We drop the subscript $\theta$ and $\theta_k$ for $\nabla$ when it is clear from context that the gradient is taken with respect to $\theta$ and evaluated at $\theta_k$.
For a compact convex set $\Theta \subset\mathbb R^p$, \(J\) is continuously differentiable on a neighborhood of \(\Theta\) according to our task definition. 
The constraint \(C\) is twice continuously differentiable. 
We also denote $B_C \coloneq \sup_{\theta\in\Theta} \|\nabla C(\theta)\|$, $H_C \coloneq \sup_{\theta\in\Theta} \|\nabla^2C(\theta)\|$ and $B_J \coloneq \sup_{\theta\in\Theta} \|\nabla J(\theta)\|$.

\subsection{Proof of the Main Theorem}

We first restate Theorem 1 below which now includes the detailed formulation.

\begin{theorem}
    Given the update rule from Eq.~(\ref{eq_app:update_primal}) and (\ref{eq_app:update_dual}) at iteration step $k\in\mathbb N_0$, we use a sequence \(\{\eta_k^{\theta}\}\) and \(\{\eta_k^{\lambda}\}\) that satisfy
    \begin{equation}
        \label{eq:step_cond_1}
        \eta_k^{\theta} > 0, \;\;\; \eta_k^{\lambda} > 0,
    \end{equation}
    \begin{equation}
        \label{eq:step_cond_2}
        \eta_k^{\theta}\to0, \;\;\; \eta_k^{\lambda}\to0,
    \end{equation}
    \begin{equation}
        \label{eq:step_cond_3}
        \sum_{k=0}^{\infty}\eta_k^{\theta}=\infty, \;\;\;\sum_{k=0}^{\infty}\eta_k^{\lambda}=\infty,
    \end{equation}
    and
    \begin{equation}
        \label{eq:step_cond_4}
        \bar\rho_n \coloneq \sup_{k\ge n} \frac{\eta_k^{\lambda}}{\eta_k^{\theta}} \rightarrow0.
    \end{equation}
    In addition,
    \begin{equation}
        \label{eq:step_cond_5}
        \eta_k^{\theta} \left( 1+\sum_{j=0}^{k-1}\eta_j^{\lambda} \right) \rightarrow0.
    \end{equation}
    Assume that there exists a compact convex set $\Theta \subset\mathbb R^p$ such that $\forall k$, $\theta_k\in\Theta$, and there exist constants $\delta>0$ and $\nu>0$ such that $\tilde{C}(\theta)\ge-\delta \;\; \Rightarrow \;\; \|\nabla C(\theta)\|\ge\nu$, then during optimization under Eq.~(\ref{eq_app:update_primal}) and (\ref{eq_app:update_dual}), $\lambda_k$ remains bounded $\lambda_k\in[0,\Lambda_0]$, and we can obtain an asymptotically feasible policy such that $\forall\varepsilon>0$, $\forall n,$ $\exists k\ge n: C(\theta_k)\le\alpha+\varepsilon$.
\end{theorem}

\begin{proof}

We begin by proving that during optimization the sequence $\{\lambda_k\}$ is bounded. Due to the property of LME, $0<C(\theta)<c_{\max}$. Hence, $\forall\theta\in\Theta$,
\begin{equation}
    |\tilde{C}(\theta)|\le \max\{\alpha,c_{\max}-\alpha\} \coloneqq c_{\max}^\alpha.
\end{equation}
By Eq.~(\ref{eq_app:update_dual})
\begin{equation}
    \lambda_{k+1} = [\lambda_k + \eta_k^{\lambda} \tilde{C}(\theta_k)]_+ \le \lambda_k+\eta_k^{\lambda} c_{\max}^\alpha.
\end{equation}
Thus,
\begin{equation}
    \lambda_k \le \lambda_0+ c_{\max}^\alpha \sum_{j=0}^{k-1} \eta_j^{\lambda}
\end{equation}
Multiplying by \(\eta_k^{\theta}\) we have
\begin{equation}
    \eta_k^{\theta}\lambda_k \le \eta_k^{\theta}\lambda_0 + c_{\max}^\alpha \eta_k^{\theta}\sum_{j=0}^{k-1}\eta_j^{\lambda}
\end{equation}
By Eq.~(\ref{eq:step_cond_5}) we have
\begin{equation}
    \label{eq_app:converg_step_dual}
    \eta_k^{\theta}\lambda_k \to 0.
\end{equation}
Let $\Lambda_0> \nicefrac{2B_J}{\nu}$. Consider an iteration for which $\lambda_k\ge\Lambda_0$ and $\tilde{C}(\theta_k)\ge-\delta$, by Taylor's theorem,
\begin{equation}
    \label{eq_app:taylor}
    C(\theta_{k+1}) \le C(\theta_k) + \nabla C(\theta_k)^\top\Delta\theta_k + \frac{H_C}{2} \|\Delta\theta_k\|^2,
\end{equation}
where
\begin{align}
    \Delta\theta_k \coloneq{}& \theta_{k+1}-\theta_k \\
                   ={} & -\eta_k^{\theta} \left( \nabla J(\theta_k) + \lambda_k\nabla C(\theta_k) \right). \label{eq_app:theta_diff}
\end{align}
Substituting Eq.~(\ref{eq_app:theta_diff}) into Eq.~(\ref{eq_app:taylor}) gives:
\begin{equation}
\begin{split}
    \label{eq_app:c_diff}
    \tilde{C}(\theta_{k+1})-\tilde{C}(\theta_k) \le{}& -\eta_k^{\theta} \nabla C^\top\nabla J \\
    &- \eta_k^{\theta}\lambda_k \|\nabla C\|^2 \\
    &+ \frac{H_C}{2} (\eta_k^{\theta})^2 \| \nabla J+\lambda_k\nabla C \|^2,
\end{split}
\end{equation}
The first two terms in Eq.~(\ref{eq_app:c_diff}) satisfy
\begin{align}
    -\nabla C^\top\nabla J - \lambda_k \|\nabla C\|^2 &\le B_J\|\nabla C\| - \lambda_k\|\nabla C\|^2 \\
    & = -\|\nabla C\| \left( \lambda_k\|\nabla C\|-B_J \right). \label{eq_app:c_diff_first2}
\end{align}
by Cauchy-Schwarz $-\nabla C^\top\nabla J \le B_J\|\nabla C\|$.
Since we are at an iteration $k$ where \(\tilde{C}(\theta_k)\ge-\delta\), \(\lambda_k\ge\Lambda_0\) and $\lambda_k\nu>2B_J$, by assumption $\|\nabla C(\theta_k)\|\ge\nu$.
Hence,
\begin{align}
    \lambda_k\|\nabla C\| - B_J & \ge \lambda_k\|\nabla C\| - \frac12\lambda_k\nu \\
    & \ge \lambda_k\|\nabla C\| - \frac12\lambda_k\|\nabla C\| \\
    & = \frac12\lambda_k\|\nabla C\|.
\end{align}
Substituting this into Eq.~(\ref{eq_app:c_diff_first2}) gives
\begin{align}
    B_J\|\nabla C\| - \lambda_k\|\nabla C\|^2 & \le -\frac12 \lambda_k \|\nabla C\|^2 \\
    & \le -\frac12 \lambda_k\nu^2.
\end{align}
The third term in Eq.~(\ref{eq_app:c_diff}) satisfy $\| \nabla J+\lambda_k\nabla C \| \le B_J+\lambda_kB_C$.
Thus,
\begin{equation}
\begin{split}
    \label{eq_app:c_diff_bound}
    \tilde{C}(\theta_{k+1})-\tilde{C}(\theta_k) \le{}& -\frac12 \eta_k^{\theta}\lambda_k\nu^2 \\
    &+ \frac{H_C}{2} (\eta_k^{\theta})^2 (B_J+\lambda_kB_C)^2.
\end{split}
\end{equation}%
Divide the second term in Eq.~(\ref{eq_app:c_diff_bound}) by \(\eta_k^{\theta}\lambda_k\) we can get
\begin{align}
    \frac{ \frac{H_C}{2}(\eta_k^{\theta})^2(B_J+\lambda_kB_C)^2 }{ \eta_k^{\theta}\lambda_k } &= \frac{H_C}{2} \eta_k^{\theta} \frac{ (B_J+\lambda_kB_C)^2 }{ \lambda_k } \\
    &= \frac{H_C}{2} \eta_k^{\theta} \left( \frac{B_J^2}{\lambda_k} + 2B_JB_C + \lambda_kB_C^2 \right) \\
    &\le M\eta_k^{\theta} \left( \frac1{\lambda_k} + 1+ \lambda_k \right) \rightarrow 0,
\end{align}
since $M$ is a constant and$\nicefrac{\eta_k^{\theta}}{\lambda_k}\to0$, $\eta_k^{\theta}\to0$ and $\eta_k^{\theta}\lambda_k\to0$ according to Eq.~(\ref{eq_app:converg_step_dual}).
Therefore, there exists an index \(K_1\) such that for all
\(k\ge K_1\), \(\lambda_k\ge\Lambda_0\) and \(\tilde{C}(\theta_k)\ge-\delta\),
\begin{align}
    \frac{H_C}{2} (\eta_k^{\theta})^2 (B_J+\lambda_kB_C)^2 &\le \epsilon\eta_k^{\theta}\lambda_k \\
    &\le \frac14 \eta_k^{\theta}\lambda_k\nu^2,
\end{align}
by choosing \(\epsilon=\nicefrac{\nu^2}{4}\). As a result, Eq.~(\ref{eq_app:c_diff_bound}) gives
\begin{align}
    \tilde{C}(\theta_{k+1})-\tilde{C}(\theta_k) &\le -\frac14\eta_k^{\theta}\lambda_k\nu^2 \\
    &\le -\frac14\Lambda_0\eta_k^{\theta}\nu^2. \label{eq_app:constraint_diff_bound}
\end{align}
In word, whenever the multiplier is sufficiently large and the policy is unsafe or near the boundary, the primal update strictly decreases the constraint residual.

We next show that individual changes in \(\tilde{C}\) become arbitrarily small. Eq.~(\ref{eq_app:theta_diff}) can be bounded by
\begin{equation}
    \|\Delta\theta_k\| \le \eta_k^{\theta}(B_J+\lambda_kB_C) \rightarrow 0,
\end{equation}
by Eq.~(\ref{eq:step_cond_2}) and~(\ref{eq_app:converg_step_dual}).
From Eq.~(\ref{eq_app:taylor}) we then have
\begin{align}
    |C(\theta_{k+1})-C(\theta_k)| \le B_C\|\Delta\theta_k\| + \frac{H_C}{2} \|\Delta\theta_k\|^2 \rightarrow 0.
\end{align}
Choose \(K_2\ge K_1\) so large that $\forall k\ge K_2$,
\begin{equation}
    \label{eq_app:converg_constraint}
    |\tilde{C}(\theta_{k+1})-\tilde{C}(\theta_k)| < \frac{\delta}{2}.
\end{equation}

Now consider an excursion of the multiplier above \(\Lambda_0\). Suppose at some \(n\ge K_2\), $\lambda_n\ge\Lambda_0$. First suppose $\tilde{C}(\theta_n)\le0$. If $-\delta\le \tilde{C}(\theta_n)\le0$, Eq.~(\ref{eq_app:constraint_diff_bound}) gives
\begin{equation}
    \tilde{C}(\theta_{n+1})<\tilde{C}(\theta_n)\le0.
\end{equation}
If $\tilde{C}(\theta_n)<-\delta$ happens, then Eq.~(\ref{eq_app:converg_constraint}), which does not depend on $\tilde{C}(\theta_n)\ge-\delta$, gives
\begin{equation}
    \tilde{C}(\theta_{n+1})<\tilde{C}(\theta_n)+\frac{\delta}{2}<-\frac{\delta}{2}<0.
\end{equation}
In either case, as long as $\lambda_k\ge\Lambda_0$, we have $\tilde{C}(\theta_{n})<0$ and consequently $\lambda_{k+1}\le\lambda_k$. So once \(\tilde{C}(\theta_k)\le0\), the multiplier cannot increase while it remains above \(\Lambda_0\). Now suppose instead that $\tilde{C}(\theta_n)>0$. Define $m \coloneq \min \{ j\ge n: \tilde{C}(\theta_j)\le0 \}$. When \(n\le j<m\), $\tilde{C}(\theta_j)>0$. So the multiplier cannot decrease and therefore $\lambda_j\ge\Lambda_0.$ From Eq.~(\ref{eq_app:constraint_diff_bound}) we have
\begin{align}
    \tilde{C}(\theta_{m}) &= \tilde{C}(\theta_n) + \sum_{j=n}^{m-1} \left( \tilde{C}(\theta_{j+1})-\tilde{C}(\theta_j) \right) \\
    &\le \tilde{C}(\theta_n) - \sum_{j=n}^{m-1} \frac14\Lambda_0\eta_j^{\theta}\nu^2
\end{align}
By Eq.~(\ref{eq:step_cond_3}), $\sum_{j=n}^{\infty}\eta_j^{\theta}=\infty$. Hence, there must exist a finite \(m\) such that $\tilde{C}(\theta_m)\le0$ and $\tilde{C}(\theta_{m-1})>0$. Similarly, we have
\begin{equation}
    \label{eq_app:c_diff_decreasing}
    0< \tilde{C}(\theta_{m-1}) \le \tilde{C}(\theta_n) - \frac14\Lambda_0\nu^2 \sum_{j=n}^{m-2} \eta_j^{\theta}.
\end{equation}
Recall that in this subcase we have supposed $\tilde{C}(\theta_n)>0$, so
\begin{equation}
    \tilde{C}(\theta_n) \le c_{\max}-\alpha \coloneqq c_{\max+}^\alpha.
\end{equation}
Combining this with Eq.~(\ref{eq_app:c_diff_decreasing}) gives
\begin{equation}
    \sum_{j=n}^{m-2} \eta_j^{\theta} \le \frac{4c_{\max+}^\alpha}{\Lambda_0\nu^2}.
\end{equation}
As a result,
\begin{equation}
    \sum_{j=n}^{m-1} \eta_j^{\theta} \le \frac{4c_{\max+}^\alpha}{\Lambda_0\nu^2} + \bar \eta^\theta,
\end{equation}
where $\bar \eta^\theta:=\sup_{j\ge n}\eta_j^{\theta}\to0.$
During this positive-residual interval,
\begin{align}
    \lambda_m-\lambda_n &= \sum_{j=n}^{m-1} \eta_j^{\lambda}\tilde{C}(\theta_j) \\
    &\le c_{\max+}^\alpha \sum_{j=n}^{m-1}\eta_j^{\lambda} \\
    &\le c_{\max+}^\alpha\bar\rho_n \sum_{j=n}^{m-1} \eta_j^{\theta} \\
    &\le c_{\max+}^\alpha\bar\rho_n \left( \frac{4c_{\max+}^\alpha}{\Lambda_0\nu^2} + \bar \eta^\theta \right).
\end{align}
By Eq.~(\ref{eq:step_cond_4}) the right-hand side tends to zero. Thus, an excursion above \(\Lambda_0\) can increase the multiplier only by a quantity tending to zero before the constraint becomes satisfied. Once \(\tilde{C}(\theta_m)\le0\), the preceding argument shows that \(\lambda_k\) cannot continue increasing while it stays above \(\Lambda_0\). Consequently, all sufficiently late excursions above \(\Lambda_0\) have uniformly bounded height. The finitely many earlier iterates also have a finite maximum. Hence,
\begin{equation}
    \sup_k\lambda_k<\infty.
\end{equation}

Then based on the dual recursion at Eq.~(\ref{eq_app:update_dual}), we have $\lambda_{k+1} \ge \lambda_k+\eta_k^{\lambda}\tilde{C}(\theta_k)$. Summing from an arbitrary iteration \(n\) to \(N-1\), where \(N>n\),
\begin{align}
    \sum_{k=n}^{N-1} \eta_k^{\lambda} \tilde{C}(\theta_k) &\le \lambda_N-\lambda_n \\
    &\le \Lambda_0.
\end{align}
Let $S_{n,N} \coloneq \sum_{k=n}^{N-1} \eta_k^{\lambda}$, then
\begin{equation}
    \frac{ \sum_{k=n}^{N-1} \eta_k^{\lambda} \tilde{C}(\theta_k) }{ S_{n,N} } \le \frac{\Lambda_0}{S_{n,N}}.
\end{equation}
The minimum of finitely many numbers is no greater than their weighted average, so
\begin{equation}
    \min_{n\le k<N} \tilde{C}(\theta_k) \le \frac{\Lambda_0}{S_{n,N}}.
\end{equation}
By Eq.~(\ref{eq:step_cond_3}), for every fixed \(n\), \(S_{n,N}\to\infty\) as \(N\to\infty\). Hence,
\begin{equation}
    \limsup_{N\to\infty} \min_{n\le k<N} \tilde{C}(\theta_k) \le0,
\end{equation}
and equivalently,
\begin{equation}
    \liminf_{k\to\infty} \bigl[C(\theta_k)-\alpha\bigr] \le0
\end{equation}
Formally,
\begin{equation}
    \forall\varepsilon>0,\ \forall n,\ \exists k\ge n: C(\theta_k)\le\alpha+\varepsilon.
\end{equation}
Therefore, for every tolerance \(\varepsilon>0\) and every arbitrarily late iteration index \(n\), there exists some \(k\ge n\) such that $C(\theta_k)\le\alpha+\varepsilon$ and the theorem is proved.

\end{proof}

\end{document}